\documentclass[12pt]{article}

\usepackage{amsmath}
\usepackage{times}
\usepackage[dvipsone]{graphicx}
\usepackage{color}
\usepackage{multirow}
\usepackage[authoryear]{natbib}
\usepackage{rotating}
\usepackage{bbm}
\usepackage{latexsym}
\usepackage{epstopdf}   

\makeatletter
\@addtoreset{equation}{section}
\makeatother
\renewcommand{\theequation}{\arabic{section}.\arabic{equation}}

\newcommand{\captionfonts}{\normalsize}

\makeatletter
\long\def\@makecaption#1#2{%
  \vskip\abovecaptionskip
  \sbox\@tempboxa{{\captionfonts #1: #2}}%
  \ifdim \wd\@tempboxa >\hsize
    {\captionfonts #1: #2\par}
  \else
    \hbox to\hsize{\hfil\box\@tempboxa\hfil}%
  \fi
  \vskip\belowcaptionskip}
\makeatother

\usepackage{amsmath}
\usepackage{mathrsfs}

\usepackage{algorithm}
\usepackage{algorithmic}
\makeatletter
\renewcommand{\fnum@algorithm}{\fname@algorithm}
\makeatother

\usepackage{amsthm}

\usepackage{pmat}
\usepackage{subfig}
\usepackage{caption}

\usepackage{amssymb}

\usepackage[retainorgcmds]{IEEEtrantools}

\renewcommand{\theequationdis}{{\normalfont (\theequation)}}
\renewcommand{\theIEEEsubequationdis}{{\normalfont (\theIEEEsubequation)}}

\makeatletter
\newcommand*{\rom}[1]{\expandafter\@slowromancap\romannumeral #1@}
\makeatother

\newtheorem{thm}{Theorem}
\newtheorem{lem}{Lemma}
\newtheorem{dfn}{Definition}
\newtheorem{prp}{Proposition}
\newtheorem*{rmk}{Remark}
\newtheorem{rmk-2}{Remark}
\newtheorem{rmk-3}{Remark}
\newtheorem{rmk-4}{Remark}
\newtheorem{rmk-5}{Remark}
\newtheorem{rmk-6}{Remark}
\newtheorem{rmk-7}{Remark}
\newtheorem{rmk-8}{Remark}
\newtheorem{cl}{Corollary}

\usepackage[hyperindex,breaklinks]{hyperref}
\hypersetup{
           breaklinks=true,   
           colorlinks=true,   
           pdfusetitle=true,  
        }
\usepackage{setspace}
\begin{document}

\ \vspace{20mm}\\

{\LARGE \flushleft On an Interpretation of ResNets via Solution Constructions}

\ \\
{\bf \large Changcun Huang}\\
{cchuang@mail.ustc.edu.cn}\\
%


\thispagestyle{empty}
\markboth{}{NC instructions}
\ \vspace{-0mm}\\
%
\begin{center} {\bf Abstract} \end{center}
This paper first constructs a typical solution of ResNets for multi-category classifications by the principle of gate-network controls and deep-layer classifications, from which a general interpretation of the ResNet architecture is given and the performance mechanism is explained. We then use more solutions to further demonstrate the generality of that interpretation. The universal-approximation capability of ResNets is proved.


\ \\[-2mm]

\section{Introduction}
\citet*{He2016a} introduced a type of shortcut connection in the architecture of a feedforward neural network, which has been proved effective in the learning of particularly deep neural networks. The modified architecture is called \textsl{residual network} (ResNet), which is widely applied and nearly becomes a standard component of network architectures, such as in Transformer \citep*{Vaswani2017}. Note that in \citet*{He2016b}, another proposed shortcut connection is slightly different from that of \citet*{He2016a} in whether or not a ReLU is used after an addition operation. Both of the above two versions are called \textsl{ResNet} and this paper will study the former one.

\subsection{Related Work}
The underlying rationale of shortcut connections had been investigated from several aspects. \citet*{Van Der Smagt1998} considered that the shortcut connection could solve the singularity problem of Hessian matrices for the training. \cite{Srivastava2015} used this architecture to regulate the information flow to enhance the training of deep neural networks, whose thought comes from the famous LSTM \citep*{Hochreiter1997}. Both of the above two ideas are related to the computational details of the learning.

\cite{He2016a} observed that an identity map may not be easily realized by a deep neural network, for which the shortcut connection is added to the architecture. \cite{He2016b} also analysed the error or information propagation of ResNets from the viewpoint of the training.

\citet*{Schraudolph1998} suggested that the shortcut connection is helpful to normalize the parameters; that is, remove the constant center of weights, which is supposed to accelerate the learning process. The author also said that by introducing the shortcut connection, the hidden-layer units could be freed from the responsibility of the ``linear moment'' and then concentrate on other part of error signals.

In \citet*{Huang2020}, in order to make an excluded category of a multi-category data set have zero outputs, the author introduced a subnetwork called ``T-bias'', which in fact contributes to a ResNet architecture. Thus, the universal-approximation results of \citet*{Huang2020} are also applicable to ResNets.

\citet*{Chen2018} made an analogy between ResNets and ordinary differential equations; although novel in its continuous-depth perspective, the model of \citet*{Chen2018} is not explicitly relevant to ResNets in its performing mechanism. \citet*{E2017} associated differential equations or dynamic systems with ResNets from a general viewpoint, in terms of the appearance of the mathematical expressions, but without mentioning the concrete details.

\subsection{Arrangement of this Paper}
The main purpose of this paper is to combine the deep-layer classfcation with the thought of the gate control of \citet*{Hochreiter1997} to explain the mechanism of ResNets.

Section 2 gives a model description of a ResNet. Section 3 constructs a typical solution for multi-category classifications, which will serve as the existence proof of the general conclusions of section 5. Sections 4 investigates the ResNet solution of \citet*{Huang2020} and applies its universal-approximation results to ResNets. Section 5 proposes a general interpretation of ResNets and the main results are summarized from section 3. Section 6 uses more solutions to further demonstrate the generality of the interpretation of section 5. Section 7 is the summary.

Throughout this paper, the units of a neural network are the type of rectified linear unit (ReLU); and when the output of a unit is positive, we say that it is activated. The cardinality of a data set $D$ of $n$-dimensional space is assumed to be finite.

\section{Model Description}
This section uses the gate-control idea \citep*{Hochreiter1997} to model the ResNet, on the basis of which a typical solution will be constructed in section 3.

\begin{dfn}
A ResNet block is a neural network that has at least three layers with the following constraints: The units between the input layer and the output layer are fully connected, whose links are collectively called a shortcut connection; the connections between the units of the last hidden layer and the output layer are an one-to-one correspondence. For simplicity, the abbreviated version ``block'' will also be used in this paper.
\end{dfn}

\begin{rmk}
The restriction above to the one-to-one correspondence is for both the typicalness and the simplicity of the network architecture. The reason is that by the principle of this paper, if the fully connected mode works, we can always find a solution of the former by adding a new layer; conversely, if a solution of the former exists, it can be converted to the form of the latter by setting some weight parameters to be zero.
\end{rmk}

\begin{dfn}
Denote a ResNet block by $\mathcal{B}$, whose input layer and output layer are $n$-dimensional and $m$-dimensional, respectively. We call the subnetwork of the hidden layers of $\mathcal{B}$ a gate network, which together with the input layer realize a function
\begin{equation}
f: \mathbb{R}^n \to \mathbb{R}^{m},
\end{equation}
which is called the gate function of $\mathcal{B}$. The size $m \times n$ weight matrix $\boldsymbol{W}$ of the shortcut connection is called the shortcut matrix.
\end{dfn}

\begin{figure}[!t]
\captionsetup{justification=centering}
\centering
\includegraphics[width=1.5in, trim = {4.0cm 3.2cm 6.0cm 3cm}, clip]{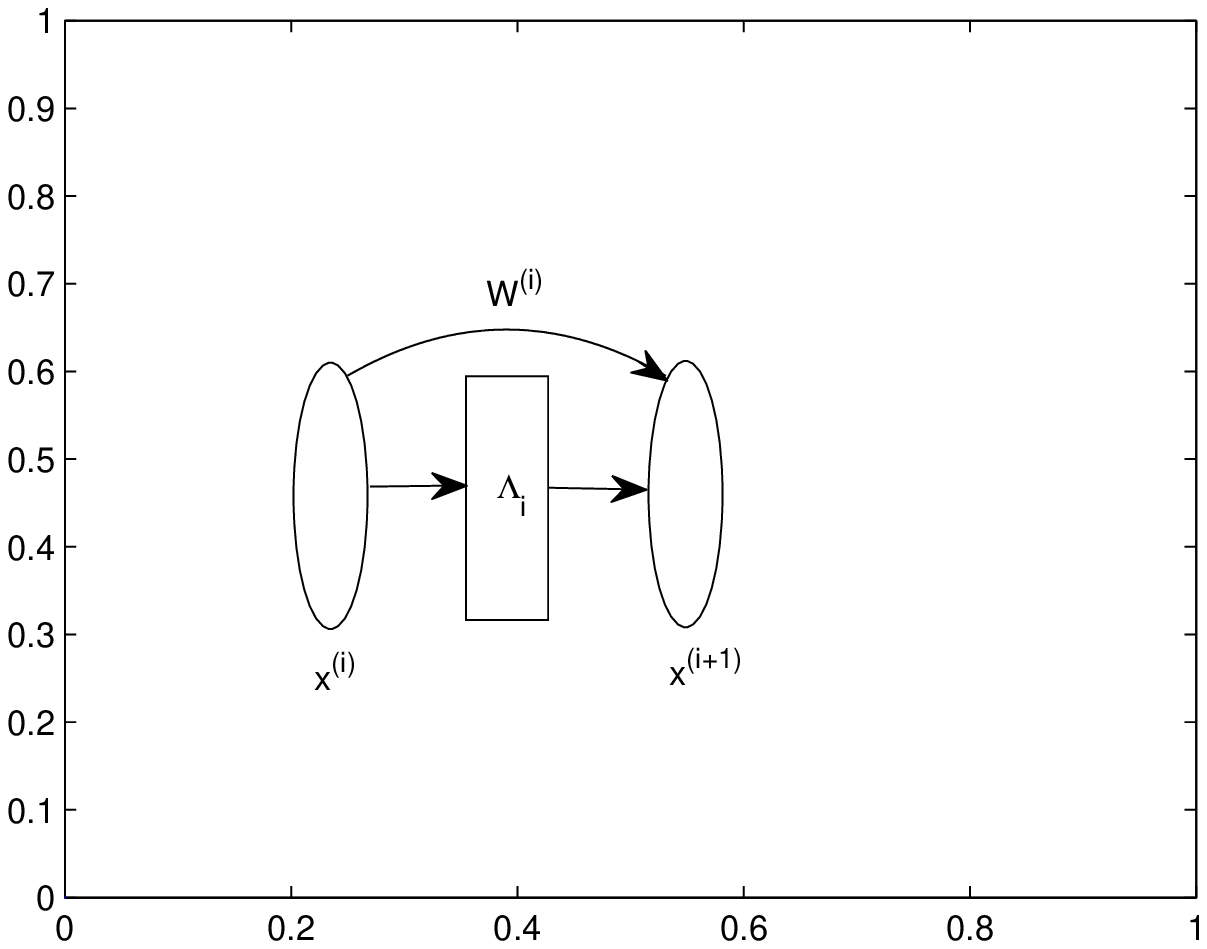}
\caption{A ResNet block.}
\label{Fig.1}
\end{figure}

\noindent
\textbf{Example 1}. Figure \ref{Fig.1} is a ResNet block. The input and output vectors are $\boldsymbol{x}^{(i)}$ and $\boldsymbol{x}^{(i+1)}$, respectively. The arc between $\boldsymbol{x}^{(i)}$ and $\boldsymbol{x}^{(i+1)}$ represents the shortcut connection, with $\boldsymbol{W}^{(i)}$ as the shortcut matrix. The square marked by $\Lambda_i$ is the gate network.

By the example above, we can represent a ResNet block by function
\begin{equation}
\boldsymbol{y} = \mathcal{F}(\boldsymbol{x}, \Lambda, \boldsymbol{\alpha}, \boldsymbol{W}, \boldsymbol{b}),
\end{equation}
where $\boldsymbol{x}$ is the $n$-dimensional input, $\boldsymbol{y}$ is the $m$-dimensional output, $\Lambda$ is the gate network, $\boldsymbol{\alpha}$ of size $m \times 1$ is the output-weight vector of $\Lambda$, $\boldsymbol{W}$ is the shortcut matrix, and $\boldsymbol{b}$ is the size $m \times 1$ bias vector of the output layer, or equivalently,
\begin{equation}
\boldsymbol{y} = \mathcal{F}(\boldsymbol{x}, f(\boldsymbol{x}), \boldsymbol{\alpha}, \boldsymbol{W}, \boldsymbol{b}),
\end{equation}
where $f(\boldsymbol{x})$ is the gate function of equation 2.1.

\vspace{3.0mm}
\noindent
\textbf{Example 2}. If $n = m$ and $\boldsymbol{W}$ is an $n \times n$ identity matrix, this is a case of \citet*{He2016a} when dimensionality augmentation is not involved in the output layer.

\begin{dfn}
We call a neural network a ResNet, if it is obtained by the concatenation of ResNet blocks $B_i$'s for $i = 1, 2,\cdots, \rho-1$ with $\rho \ge 3$, that is, the output layer of $B_i$ is the input layer of $B_{i+1}$. The input layers and output layers of $B_i$'s comprise the layers of the ResNet whose depth is $\rho$, without considering the gate networks. The width of a layer of a ResNet is the number of its units.
\end{dfn}

Let $\mathcal{R}$ be a ResNet described in definition 3. Then the depth of $\mathcal{R}$ is $\rho$, and denote by $n_j$ the number of the units of layer $j$ for $j = 1, 2, \cdots, \rho$, with $n_1 = n$ and $n_{\rho} = m$ for the input layer and output layer of $\mathcal{R}$, respectively.

To ResNet $\mathcal{R}$, the output vector $\boldsymbol{x}^{(i+1)}$ of the $i$th block $B_i$ or the $i+1$th layer can be expressed as
\begin{equation}
\boldsymbol{x}^{(i+1)} = \sigma(\boldsymbol{s}_{i+1})
\end{equation}
with
\begin{equation}
\boldsymbol{s}_{i+1} = \boldsymbol{W}^{(i)}\boldsymbol{x}^{(i)} + \boldsymbol{b}^{(i)} + \boldsymbol{\alpha}_i \circ f^{(i)}(\boldsymbol{x}^{(i)}),
\end{equation}
where matrix $\boldsymbol{W}^{(i)}$ of size $n_{i+1} \times n_{i}$ is the shortcut matrix of $B_i$, $\boldsymbol{b}^{(i)}$ is the size $n_{i+1} \times 1$ bias vector of the output layer, $f^{(i)}$ is the gate function with $\boldsymbol{\alpha}_i$ being its output-weight vector and $\circ$ denoting the Hadamard or element-wise matrix product, and
\begin{equation}
\sigma(s) = \max(0, s)
\end{equation}
is the activation function of a ReLU; when $s$ of equation 2.6 is a vector, it means that each of its entries is manipulated by the operator $\sigma$.

Using the notation of equation 2,2, equations 2.4 and 2.5 can be combined into
\begin{equation}
\boldsymbol{x}^{(i+1)} = \mathcal{F}_i(\boldsymbol{x}^{(i)}, \Lambda_i, \boldsymbol{\alpha}_i, \boldsymbol{W}^{(i)}, \boldsymbol{b}^{(i)})
\end{equation}
for $i = 1, 2, \cdots, \rho-1$, which can be regarded as the representation of the ResNet $\mathcal{R}$ above. Equation 2.7 is the model that we propose for ResNets.

\begin{dfn}
We use the notation $\prod_{\nu = 1}^{\rho}n_{\nu}$ to represent the architecture of the ResNet $\mathcal{R}$ of equations 2.7, which contains the information of the depth $\rho$ and the width $n_{\nu}$ of each layer.
\end{dfn}

\section{Typical-Solution Construction}
The solution to be constructed in this section is said to be typical in the sense that the associated gate network is the simplest case that has only one layer, and that the shortcut connection uses the simplest identity map to transmit the data. The constructed solution will be the basis of the interpretation of ResNets in later sections.

\begin{dfn}
To a ResNet block of equation 2.2 with $n$-dimensional input, if its shortcut matrix is an $n \times n$ identity one, we call it an $n$-identity block, whose input and output are both $n$-dimensional. By equations 2.4 and 2.5, an $n$-identity block can be expressed as
\begin{equation}
\boldsymbol{y} = \sigma(\boldsymbol{x} + \boldsymbol{b} +  \boldsymbol{\alpha} \circ f(\boldsymbol{x})).
\end{equation}
\end{dfn}

\vspace{3.0mm}
\noindent
\textbf{Example}. The example of Figure \ref{Fig.2} is a $2$-identity block.

\begin{figure}[!t]
\captionsetup{justification=centering}
\centering
\includegraphics[width=1.9in, trim = {4.5cm 5.0cm 4.9cm 2.5cm}, clip]{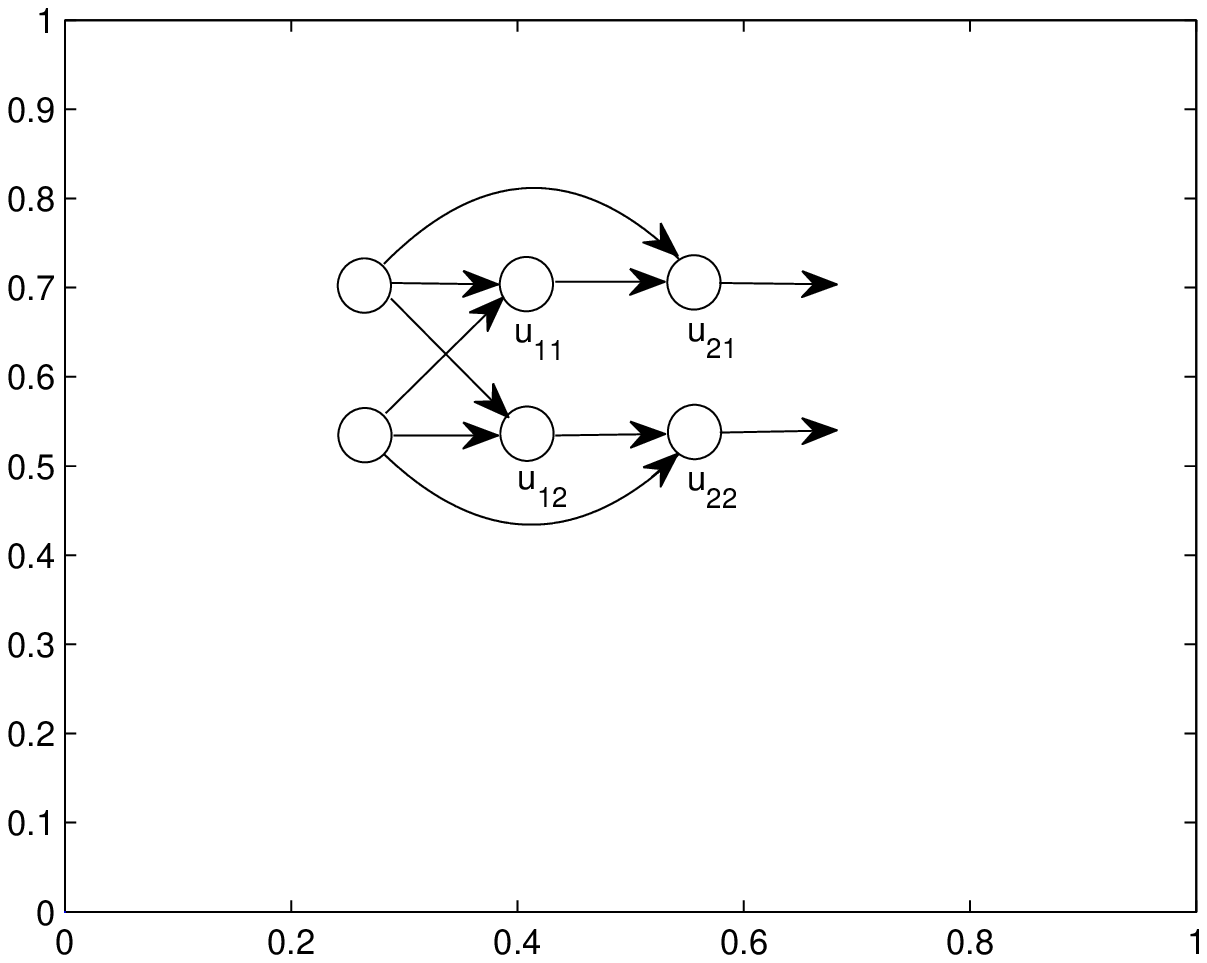}        
\caption{A $2$-identity block.}
\label{Fig.2}
\end{figure}

\begin{dfn}
A simplest $n$-block is an $n$-identity one whose gate network has only one layer. A ResNet obtained by the concatenation of simplest $n$-blocks as definition 3 is called a simplest ResNet.
\end{dfn}

Let $l$ be an $n-1$-dimensional hyperplane of $n$-dimensional space, corresponding to a unit of a neural network. We use the notations $l^+$ and $l^0$ to represent the two parts of $n$-dimensional space separated by $l$, whose outputs of the associated unit are positive and zero, respectively.

\begin{lem}
A simplest $n$-block can classify a linearly separable two-category data set $D$ of the $n$-dimensional input space. One of the category could pass through the block in the sense of affine transforms, and the other one could be excluded in the form of a zero-vector output.
\end{lem}
\begin{proof}
The proof begins with an example of Figure \ref{Fig.2}, which is a simplest 2-block denoted by $\mathcal{B}$, that is, a $2$-identity block whose gate network has only one layer. Let $D = D_1 \cup D_2$, where $D_1$ and $D_2$ correspond to the two categories of $D$, respectively. We want $D_1$ to be transmitted to the output and $D_2$ to be excluded. Let $l_{11}$ and $l_{12}$ be the lines ($n-1$-dimensional hyperplane with $n=2$) derived from the units $u_{11}$ and $u_{12}$ of the gate network, respectively. Construct $l_{11}$ such that $D_1 \subset l_{11}^0$ and $D_2 \subset l_{11}^+$; and similarly construct $l_{12}$ with the same classification effect as $l_{11}$, that is, $D_1 \subset l_{12}^0$ and $D_2 \subset l_{12}^+$.

Let $\boldsymbol{x}_1 \in D_1$ be an arbitrary element of $D_1$. By the construction above, the outputs of $u_{11}$ and $u_{12}$ of the gate network with respect to $\boldsymbol{x}_1$ are both zero. Then by equation 3.1, the output vector of $u_{21}$ and $u_{22}$ is $\sigma({\boldsymbol{x}_1} + \boldsymbol{b})$, where $\boldsymbol{b}$ is the bias vector of the output layer of $\mathcal{B}$. If the entries of $\boldsymbol{b}$ are all positive and large enough, then $\boldsymbol{x}_1$ could be transmitted to the output in the form of $\boldsymbol{{x}}_1 + \boldsymbol{b}$, which is an affine transform of the input $\boldsymbol{{x}}_1$. We can set the vector $\boldsymbol{b}$ such that all the elements of $D_1$ could pass through the block $\mathcal{B}$ as $\boldsymbol{{x}}_1$.

When the input is $\boldsymbol{x}_2 \in D_2$, the output of $u_{11}$ is nonzero. Set the output weight $\alpha_1$ of $u_{11}$ to be a negative number whose absolute value is large enough, and then the output of $u_{21}$ with respect to $\boldsymbol{x}_2$ could be zero. Since the cardinality $|D_2|$ is finite, we can find a value of $\alpha_1$ to fulfil all the elements of $D_2$, such that they all have zero output of $u_{21}$. The case of $u_{12}$ and $u_{22}$ is similar.

The proof of the general case of simplest $n$-blocks is trivial by the two-dimensional example above.
\end{proof}

\begin{dfn}
Data disentangling means that a linearly inseparable multi-category data set becomes linearly classified after passing through a neural network.
\end{dfn}

The term ``facet'' below can be found in \citet*{Grunbaum2003}, which is a face of an $n$-dimensional polytope with maximum dimensionality $n-1$. For example, in three-dimensional space, a facet is a two-dimensional face of a polyhedron. The facets of a polytope comprise its piecewise boundary.

The rigorous definition of an open or a closed convex polytope was given in \citet*{Huang2022}. Intuitively speaking, a closed convex polytope contains the boundary, while the open one doesn't.

\begin{prp}
Suppose that $D$ is a two-category data set of $n$-dimensional space, and that one of the category is contained in an open convex polytope $\mathcal{P}$ that the other category doesn't belong to. Then a simplest ResNet $\prod_{\nu = 1}^{\rho}n$ can disentangle $D$, with the depth $\rho$ determined by the number of the facets of $\mathcal{P}$.
\end{prp}
\begin{proof}
By lemma 1, each block could realize a binary classification, and its output could preserve one of the category in terms of affine transforms and exclude the other one by zero outputs. And we can concatenate the blocks one by one, each of which corresponds to a binary classification with the classification result preserved to the next block; then any two-category data set $D$ could be classified by a simplest ResNet through recursive binary classification.

The number of the above binary classifications is determined by that of the facets of the polytope $\mathcal{P}$ and so is the depth $\rho$.
\end{proof}

\begin{figure}[!t]
\captionsetup{justification=centering}
\centering
\subfloat[Independent-module architecture.]{\includegraphics[width=2.0in, trim = {3.5cm 1.9cm 4.0cm 1.7cm}, clip]{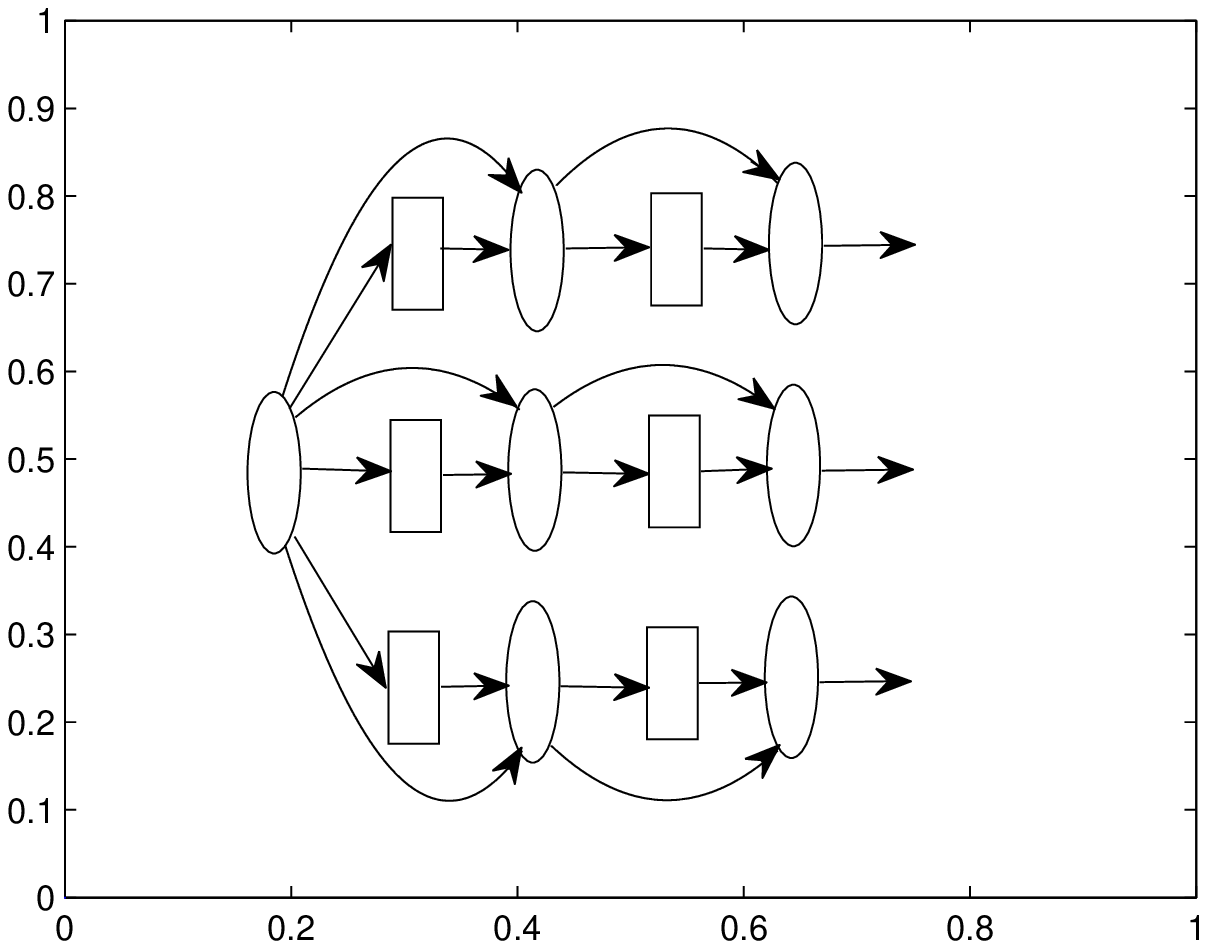}} \quad \quad \quad
\subfloat[A ResNet having projection shortcuts.]{\includegraphics[width=2.1in, trim = {3.5cm 1.9cm 2.5cm 1.7cm}, clip]{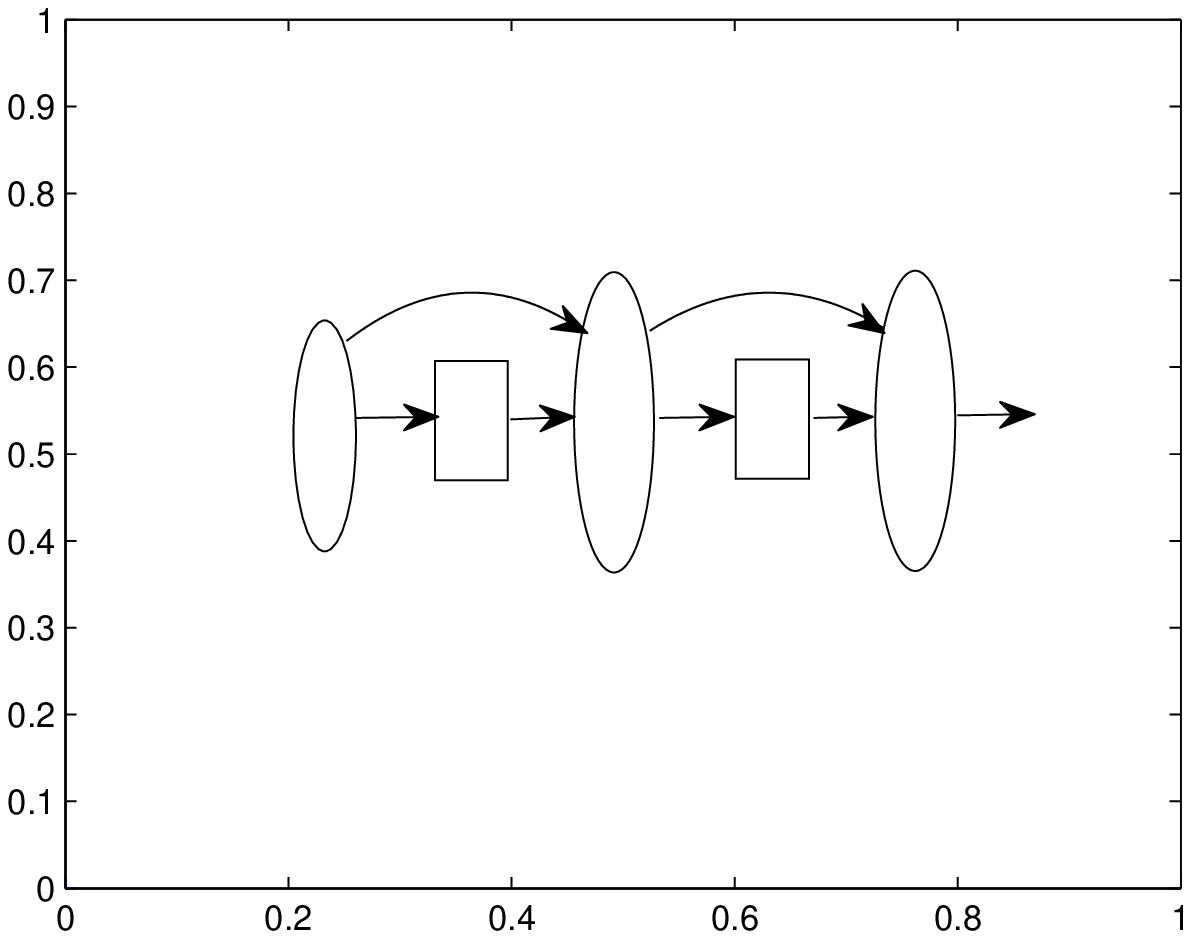}}
\caption{ResNets for multi-category classifications.}
\label{Fig.3}
\end{figure}

\begin{lem}
Let $\mathcal{N}$ be a neural network that is composed of $\nu$ simplest ResNets sharing the same $n$-dimensional input and having the same depth $\rho$. Then $\mathcal{N}$ could disentangle any two-category data set $D$ of the input space, provided that the number $\nu$ of the simplest ResNets and the depth $\rho$ of each simplest ResNet are large enough.
\end{lem}
\begin{proof}
Figure \ref{Fig.3}a is an example of network $\mathcal{N}$ with three simplest ResNets whose depths are all three. If the condition of proposition 1 is not satisfied, decompose one of its category (denoted by $*$) into $\nu$ subsets, such that each of them could be in an open convex polytope $p_i$ for $i = 1, 2, \cdots, \nu$, without containing the points of the other category (denoted by $o$).

Each simplest ResNet $R_i$ for $i = 1, 2, \cdots, \nu$ classify the $i$th subset of category $*$ by the method of proposition 1. If the depths of $R_i$'s are different, use the maximum one $\rho$; the remaining simplest ResNets add redundant blocks having the same classification result as the previous one to make the depth equal to $\rho$. Then each $R_i$ with the same depth $\rho$ would output a positive-entry vector for category $*$ and a zero vector for category $o$.

If we add a unit in a new layer of $\mathcal{N}$, set its input weights to be 1 and the bias to be 0, then it could produce positive and zero outputs for the categories $*$ and $o$, respectively. That is, $D$ is disentangled in the output layer of $\mathcal{N}$.

By the construction process, we see that the depth $\rho$ of the network $\mathcal{N}$ is determined by the number of the facets of each convex polytope $p_i$; and the width is associated with the number $\nu$ of $p_i$'s, which is equal to that of the simplest ResNets $R_i$'s. Both of the above two parameters of $\mathcal{N}$ depend on the input-data structure for classifications.
\end{proof}

\begin{prp}
The network architecture $\mathcal{N}$ proposed in lemma 2 could disentangle any multi-category data set $D$ of the $n$-dimensional input space, if the depth $\rho$ and the width-associated parameter $\nu$ are sufficiently large.
\end{prp}
\begin{proof}
Suppose that $D$ has $k$ categories, and each of them is dealt with analogous to category $*$ of the proof of lemma 2 (with other $k-1$ categories combined to be category $o$), through adding simplest ResNets that can increase the width-associated parameter $\nu$ of $\mathcal{N}$. The depth $\rho$ is then obtained similarly to that of lemma 2. We show that the constructed $\mathcal{N}$ could disentangle $D$.

In a new layer of $\mathcal{N}$, to each category $i$ for $i = 1, 2, \cdots, k$, add a unit $u_i$ to classify it, whose parameters are set as follows. The weights of the connections between $u_i$ and the units that have nonzero output with respect to category $i$ are set to be 1, and the remaining weights together with the bias are set to be 0. Then $u_i$ could only output nonzero value for category $i$, which means that $D$ has been disentangled by $\mathcal{N}$.
\end{proof}

We directly borrow some terminology from literature \citet*{He2016a} and give them formal definitions here.
\begin{dfn}
To a shortcut connection, if its shortcut matrix is an identity one, we call it an identity shortcut; otherwise, it is a projection shortcut, with the corresponding shortcut matrix called a projection matrix.
\end{dfn}

\noindent
\textbf{Example}. In Figure \ref{Fig.3}b, the shortcut connection between the first two layers is an example of projection shortcuts, where the different size of the ellipses signifies the unequal number of units.

\begin{thm}
A ResNet $\mathcal{R}$ can classify any multi-category data set $D$ of the input space, provided that the widths and the depth of $\mathcal{R}$ are sufficiently large.
\end{thm}
\begin{proof}
In the example of Figure \ref{Fig.3}a of proposition 2, the network architecture can be converted to that of Figure \ref{Fig.3}b, which is a ResNet (denoted by $R$) having projection shortcuts. The method is trivial by setting some of the weight parameters of $R$ (including the gate networks) to be zero. The general case is similar.
\end{proof}

\section{ResNet Solution of \citet*{Huang2020}}
The solutions of ReLU networks for function approximations in \citet*{Huang2020} are expressed in the form of a subnetwork called ``T-bias'', which in fact contributes to a type of ResNet architecture. The result of this section is somewhat a by-product of the architecture-feature examination, instead of particular consideration. Compared to related works \citet*{Lin2018} and \citet*{Aizawa2020}, our proof and network architecture are different from theirs.

\begin{thm}
ResNets are universal approximators via piecewise linear or constant functions.
\end{thm}
\begin{proof}
On the basis of section 3, we can easily interpret the solution of \citet*{Huang2020} by the perspective of ResNets. In \citet*{Huang2020}, the T-bias subnetworks could make up the gate network of a ResNet block, and the links between adjacent layers can be regarded as the shortcut connection; the biases are fixed to be 0.

Note that in the case such as Figure 7 of \citet*{Huang2020}, all the units of a layer share the same T-bias, which is different from the architecture of the ResNet block of definition 1 that has an one-to-one correspondence mode. However, it is trivial to modify the architecture of \citet*{Huang2020} such that each unit has its own T-bias, which could be the case of this paper.

We also need to change the independent-module architecture of \citet*{Huang2020} into the form of ResNets as in the proof of theorem 1, and the method is similar, by setting some of the weights of a ResNet to be zero.

Therefore, we could immediately apply the results of \citet*{Huang2020} to ResNets in the general form of this theorem.
\end{proof}

\begin{rmk-2}
The mechanism of the activation of units in \citet*{Huang2020} is different from that of lemma 1 of this paper, which suggests the diversity of the solutions of ResNets.
\end{rmk-2}

\begin{rmk-2}
The introduction of a T-bias subnetwork of \citet*{Huang2020} is due to the principle of certain solution constructions, which happens to yield the characteristic of a ResNet architecture. This may demonstrate the fundamentality of the mechanism of ResNets.
\end{rmk-2}

\section{General Interpretation}
We draw an interpretation of ResNets from section 3. The constructed solution of section 3 can be regarded as the existence proof of the results of this section, for which the proofs are all omitted.

The three following conclusions are immediately derived from theorem 1, which are some general descriptions of the concrete constructed solution.
\begin{cl}[Effect of gate networks]
To a ResNet block, the gate network could control the activation of a unit of the output layer for different categories of a multi-category input data set.
\end{cl}

\begin{cl}[Effect of shortcut connections]
The shortcut connections of a ResNet could directly transmit the data of the current layer to the next layer in the sense of affine transforms. And in the process of controlling the activation of a unit of the next layer, the data to be transmitted by shortcut connections can be a reference or a restriction that decides whether or not the unit is activated by the corresponding gate network.
\end{cl}

From the perspective of lemma 1, we describe corollaries 1 and 2 in terms of formulas in more details. In what follows, the term \textsl{usual unit} refers to a unit without shortcut-connection inputs. The activation function of a usual unit can be expressed as
\begin{equation}
y = \sigma(\boldsymbol{x}, \boldsymbol{w}, b),
\end{equation}
where $\boldsymbol{x}$ is the input, and $\boldsymbol{w}$ and $b$ are the weight vector and bias, respectively. Correspondingly, by the model of equation 2.7, the case of a ResNet-unit $u_j$ of a ResNet block is
\begin{equation}
y = \sigma(\boldsymbol{x}, f_j(\boldsymbol{x}), \alpha, \boldsymbol{w}_j, b),
\end{equation}
where $f_j(\boldsymbol{x})$ is the $j$th dimension of the gate function for $u_j$, $\alpha$ is the output weight of $f_j(\boldsymbol{x})$, $\boldsymbol{w}_j$ is the $j$th row of the shortcut matrix associated with $u_j$, and others are similar to those of equation 5.1.

By the proof of lemma 1, if we want the shortcut connection to transmit one dimension $x_j$ of the input $\boldsymbol{x}$ (or its affine transform, and similarly for the later case) by unit $u_j$, the output $f_j(\boldsymbol{x})$ of the gate network should be zero and we then use the positive bias $b$ to ensure that $u_j$ output $x_j + b$, which is one dimension of the affine-transform $\boldsymbol{x} + \boldsymbol{b}$ of input $\boldsymbol{x}$.

If we want $x_j$ to be excluded by $u_j$, the gate-network output $f_j(\boldsymbol{x})$ should be nonzero, and simultaneously the output weight $\alpha$ of $f_j(\boldsymbol{x})$ is set to be a negative number whose absolute value is large enough, through which the output of $u_j$ could be zero.

Therefore, the main difference of equation 5.2 of ResNets from equation 5.1 of usual networks without shortcut connections is the use of a gate network to control the activation of a unit according to the input. The corollary below further emphasizes this distinction from the perspective of the division of the network architecture in the performing mechanism.

\begin{cl}[Architecture-division principle]
Through separate architectures, which are called the shortcut connection and the gate network, respectively, a ResNet can transmit the input data in terms of affine transforms and simultaneously control the activation of a unit according to the input.
\end{cl}

\begin{rmk-3}
In comparison with a ResNet, other feedforward neural networks realize the above two operations by one shared architecture, and the solution is determined by the parameter settings of that single architecture.
\end{rmk-3}

\begin{rmk-3}
Corollaries 1, 2 and 3 manifest the main distinction of a ResNet from other types of neural networks that do not have shortcut connections.
\end{rmk-3}

The corollary that follows is also the characteristic of feedforward neural networks without shortcut connections \citep*{Huang2022}, which suggests the common feature of ResNets with other network architectures.
\begin{cl}[Effect of deep layers]
Each layer of a ResNet could preserve the classification results of the preceding layers, and the combination of the effects of all the layers can lead to the disentangling of the input data set.
\end{cl}

\begin{rmk-4}
The four corollaries of this section provide a general interpretation of the mechanism of ResNets.
\end{rmk-4}

\begin{rmk-4}
Note that the ResNet solution of \citet*{Huang2020} discussed in section 4 also obeys the rules of the four corollaries.
\end{rmk-4}

\section{Miscellaneous Solutions}
Although the four corollaries of section 5 stem from a concrete-solution construction, their generality can be further demonstrated by more examples, from which we can assess to what extent they may be related to the solution of engineering.

The gate network of a ResNet block can be an arbitrary neural network, including any feedforward one that could classify a multi-category data set. Thus, theoretically, one ResNet block is enough to do any classification, provided that its gate network is complex enough.

In \citet*{He2016a}, some gate networks (such as Figures 2 and 3) have two layers, corresponding to a three-layer network. We know that a three-layer network could classify any multi-category data set \citep*{Huang2022}, provided that the number of the units of the hidden layer is sufficiently large. Thus, this type of gate network of \citet*{He2016a} is capable of controlling the activation of a unit for any multi-category data set.

From the above viewpoint, the architecture design of a ResNet could be very flexible, and arbitrary depth could be set to achieve the goal of classifications or interpolations. However, to different kinds of gate networks, the difficulty of finding a solution by the training may be different, which is related to the usefulness of deep layers as mentioned in corollary 4.

\section{Summary}
We provided a novel perspective to interpret the mechanism of ResNets inspired by the gate-control idea of LSTM. The universality of our general conclusions may need more evidences, and their concrete manifestation may be diverse, which are to be studied in future. Since ResNets have been widely applied in engineering and science, its interpretation is crucial. We hope that the results of this paper could contribute to the understanding of ResNets.

\end{document}